\documentclass[11pt, letterpaper]{article} 
\usepackage{amsmath,amsfonts,amssymb,latexsym,amsthm,color} 
\usepackage[letterpaper,portrait]{geometry} 
\usepackage[utf8]{inputenc} 
\usepackage[T1]{fontenc}    
\usepackage{hyperref}       
\usepackage{url}            
\usepackage{booktabs}       
\usepackage{amsfonts}       
\usepackage{nicefrac}       
\usepackage{microtype}      
\usepackage{graphicx}

\usepackage{amsmath,amssymb,latexsym,amsthm,color}
\usepackage{bbm}
\newtheorem{thm}{Theorem}%
\newtheorem{defi}{Definition}%
\newtheorem{prop}{Proposition}
\newtheorem{cor}{Corollary}
\newtheorem{obs}{Observation}
\newtheorem{ex}{Example}

\def \bp{\boldsymbol{p}}
\def \bq{\boldsymbol{q}}

\def \S{\boldsymbol{\Delta}}
\def \vol{{\mu}}

\title{Optimal Confidence Regions for the Multinomial Parameter}
\date{}
\author{Matthew L.\ Malloy, Ardhendu Tripathy and Robert D.\ Nowak
\thanks{M.\ Malloy and R.\ Nowak are with the Electrical \& Computer Engineering Department at University of Wisconsin-Madison. emails: \texttt{\{mmalloy,  rdnowak\}@wisc.edu}. A.\ Tripathy is with the Computer Science Department at Missouri University of Science \& Technology. email: \texttt{astripathy@mst.edu}}}

\begin{document}
\maketitle

\begin{abstract}
Construction of tight confidence regions and intervals is central to statistical inference and decision making. 
This paper develops new theory showing minimum average volume   confidence regions for categorical data.  More precisely, consider an empirical distribution $\widehat{\bp}$ generated from $n$ iid realizations of a random variable that takes one of $k$ possible values according to an unknown distribution $\bp$. This is analogous to a single draw from a multinomial distribution.  A confidence region is a subset of the probability simplex that depends on $\widehat{\bp}$  and contains the unknown $\bp$ with a specified confidence.  This paper shows how one can construct minimum average volume confidence regions, answering a long standing question.  We also show the optimality of the regions directly translates to optimal confidence intervals of linear functionals such as the mean, implying sample complexity and regret improvements for adaptive machine learning algorithms. 
\end{abstract}


\section{Introduction}


This paper shows an optimal confidence region construction for the parameter of a multinomial distribution. The confidence regions, a generalization of the famous Clopper-Pearson confidence interval for the binomial \cite{clopper1934use}, are optimal in the sense of having minimal average volume in the probability simplex for a prescribed confidence level.  

Consider an empirical distribution $\widehat{\bp}$ generated from $n$ i.i.d.\ samples of a discrete random variable $X$ that takes one of $k$ values according to an unknown distribution $\bp$.  A confidence region for $\bp$ is a subset of the $k$-simplex that depends on $\widehat{\bp}$, and includes the unknown true distribution $\bp$ with a specified confidence.  
More precisely,  
$\mathcal{C}_{\delta}( \widehat{\bp}) \subset \Delta_{k} $ is a confidence region at confidence level $1{-} \delta$ if 
\begin{eqnarray} \label{eqn:cr}
\sup_{\bp \in \Delta_{k}} \mathbb{P}_{\bp}  \left( \bp  \not \in \mathcal{C}_{\delta}(\widehat{\bp}) \right) \leq \delta \, ,
\end{eqnarray}
where $\Delta_{k}$ denotes the $k$-simplex, and $\mathbb{P}_{\bp}(\cdot)$ is the probability measure under the multinomial parameter $\bp$.


Construction of tight confidence regions for categorical distributions is a long standing problem dating back nearly a hundred years  \cite{clopper1934use}.  The goal is to construct regions that are as small as possible, but still satisfy (\ref{eqn:cr}).  Broadly speaking, approaches for constructing confidence regions can be classified into \emph{i)} approximate methods that fail to guarantee coverage (i.e, (\ref{eqn:cr}) fails to hold for all $\bp$) and \emph{ii)} methods that succeed in guaranteeing coverage, but have excessive volume -- for example, approaches based on Sanov or Hoeffding-Bernstein type inequalities.   Recent approaches based on combinations of methods \cite{nowak2019tighter} have shown improvement through numerical experiment, but do not provide theoretical guarantees on the volume of the confidence regions. To the best of our knowledge, construction of confidence regions for the multinomial parameter that have minimal volume and guarantee coverage is an open problem.

One construction that has shown promise empirically is the \emph{level-set} approach of \cite{chafai2009confidence}. The level-set confidence regions are similar to `exact' and Clopper-Pearson\footnote{We note that `exact' and Clopper-Pearson are synonymous in many texts \cite{agresti1998approximate}.} regions \cite{clopper1934use} as they involve inverting tail probabilities, but are applicable beyond the binomial case, \emph{i.e.}, they are defined for $k>2$.   Clopper-Pearson, exact, and level-set confidence regions are closely related to statistical significance testing; the confidence region defined by these approaches is synonymous with the range of parameters over which the outcome \emph{is not statistically significant at an exact p-value of $1-\delta$}.  For a thorough discussion of these relationships in the binomial case, see \cite{agresti2003dealing, agresti1998approximate} and reference therein.

This paper proves that the \emph{level-set} confidence regions of \cite{chafai2009confidence}, which are extensions of Clopper-Pearson regions, are optimal in that they have minimal average volume among any confidence region construction.  More precisely, when averaged across either \emph{i)} the possible empirical outcomes, or \emph{ii)} a uniform prior on the unknown parameter $\bp$, the level-set confidence regions have minimal volume among any confidence region construction that satisfies the coverage guarantee.  The proof first involves showing that arbitrary confidence regions can be expressed as the inversion of a set mapping.  The level-set confidence regions are minimal in this setting by design, and the minimal average volume property follows.  As the authors of \cite{chafai2009confidence} observe through numerical experiment, the level-set confidence regions have small volume when compared with a variety of other approaches.  Indeed this observation is correct; the regions minimize average volume among \emph{any} construction of confidence regions.  

While motivation for tight confidence regions can be found across science and engineering, one motivation comes from the need for tighter confidence intervals for the mean.  Indeed, confidence intervals for functionals such as the mean, variance, or median, can be derived from confidence regions for the multinomial parameter by simply finding the range values assumed by the functional in the confidence region.  This paper also shows that the confidence regions can be used to generate confidence intervals for linear functionals that are tighter, on average, than any known constructions, including Hoeffding bounds, Kullback Leibler divergence-based bounds \cite{garivier2013}, and the empirical Bernstein bound \cite{mnih2008empirical,maurer2009empirical,audibert2009exploration}. The reason for the improved coverage is that, unlike other methods, the confidence intervals account for the shape of the distribution in the simplex.  When compared to standard confidence bounds for the mean based on Bernstein or Hoeffding's inequalities, the constructions can require several times fewer samples to achieve a desired interval width. 

Tight confidence regions and intervals are fundamental to the operation and analysis of many sequential learning algorithms, including reinforcement learning and multi-armed bandits \cite{jamieson2013finding, malloy2014sequential}, guiding both data collection and providing namesake (for example, the Upper Confidence Bounds algorithm of \cite{auer2002finite} and lil'UCB of \cite{jamieson2014lil}).  The performance of such methods hinges on the quality of sequential actions, which in turn depends critically on the width of confidence intervals. If the bounds are too loose, then such sequential algorithms may perform no better than non-adaptive or random action selection.  If they are too aggressive (i.e., invalid confidence bounds), then guarantees are null and algorithms can fail catastrophically. 
This is particularly true in the small sample regime, where sequential learning algorithms have the most to gain over non-adaptive counterparts.  To highlight these advantages, we demonstrate the confidence intervals in a multi-armed bandit setting.

Direct computation of minimal volume regions involves enumerating empirical outcomes and computing partial sums.  In the small sample regime (for example, $n=50$ and $k=5$) computation of the minimal volume regions is straightforward. As computation scales as $O(n^{k})$, this becomes prohibitive for modest $k$.  To aid in computation, we show an outer bound based on the Kullback Leibeler divergence that can be used to accelerate computation of the regions.  We also note that the large sample regime, where computation is prohibitive, is well-served by traditional confidence regions based on asymptotic statistics.


\section{Preliminaries} \label{sec:pre}
Let $X = X_1, \dots, X_n \in \mathcal{X}^n$ be a i.i.d.\ sample of a categorical random variable where $X_i$ takes one of $k$ possible values from a set of categories $\mathcal{X}$.   The empirical distribution of $X$ is the relative proportion of occurrences of each element of $\mathcal{X}$ in $X$.  More precisely, let $\mathcal{X} =: \{x_1, \dots, x_k\}$ and define ${n}_i = \sum_{j=1}^n {\mathbf{1}_{\{ X_j = x_i \} }  }$ for $i = 1,\dots k$.  Then  $\widehat{\boldsymbol{p}}(X) = [{n}_1/n, \dots, n_k/n] \in \Delta_{k,n}$, 
where $\Delta_{k,n}$ is the discrete simplex from $n$ samples over $k$ categories:
\begin{eqnarray} \nonumber
\Delta_{k,n} = \left\{ \widehat{\bp} \in \{0, 1/n, \dots, 1\}^k  :   \sum_i \widehat{p}_i =1  \right\}.
\end{eqnarray}
To simplify notation in what follows, we write $\mathbb{P}_{\boldsymbol{p}}(\widehat{\boldsymbol{p}})$ as shorthand for $\mathbb{P}_{\bp} \left(\left\{ X \in \mathcal{X}^n : \widehat{\bp}(X) = \widehat{\bp}\right\} \right)$ where $\mathbb{P}_{\bp}( \cdot)$ denotes the probability measure under $\bp \in\Delta_{k}$ and  $\Delta_k$ is the $k$-dimensional probability simplex:
\begin{eqnarray} \nonumber
\Delta_k = \left\{\bp \in [0,1]^k :  \sum_i p_i =1  \right\}.
\end{eqnarray}
We refer to the powerset of $\Delta_k$ as $\mathcal{P}(\Delta_k)$, and 
likewise, $\mathcal{P}(\Delta_{k,n})$ as the power set of $\Delta_{k,n}$.
We also write $\mathbb{P}_{\boldsymbol{p}}(\mathcal{S})$ for $\mathcal{S} \subset \Delta_{k,n}$ as shorthand for  $\mathbb{P}_{\bp} \left(\left\{ X \in \mathcal{X}^n : \widehat{\bp}(X) \in \mathcal{S}\right\} \right)$.  $\mathbb{P}_{\bp}(\widehat{\bp})$ is fully characterized by the multinomial distribution with parameter $\bp \in\Delta_{k}$:
\begin{eqnarray} \nonumber
\mathbb{P}_{\bp}(\widehat{\bp}) 
= \frac{n!}{(n\widehat{p}_1)!  \ldots (n\widehat{p}_k)!} p_1^{n\widehat{p}_1}  \cdots p_k^{n\widehat{p}_k}.
\end{eqnarray}
The parameter $\bp$ specifies the unknown distribution over $\mathcal{X}$.

The focus of this paper is construction of confidence regions for $\bp$ from a sample $X_1, \dots, X_n$.  Since $\widehat{\bp}$ is a sufficient statistic for $X_1, \dots, X_n$, we focus on construction of confidence regions that are functions of $\widehat{\bp}$ with no loss of generality.

\begin{defi} Confidence region.  Let $\mathcal{C}_{\delta}( \widehat{\bp}): \Delta_{k,n} \rightarrow \mathcal{P}(\Delta_{k})$ be a set valued function that maps an observed empirical distribution $\widehat{\bp}$ to a subset of the $k$-simplex. $\mathcal{C}_{\delta}( \widehat{\bp})$   is a \emph{confidence region} at confidence level $1-\delta$ if (\ref{eqn:cr}) holds. \label{def:cr}
\end{defi}

\begin{obs} \label{obs:cc} Equivalent Characterization via Covering Collections.  
Let  $\mathcal{S({\bp})}: \Delta_{k} \rightarrow \mathcal{P}(\Delta_{k,n})$ be given as:
\begin{eqnarray}  \label{eqn:sc_val}
\mathcal{S({\bp})} = \left\{ \widehat{ \bp} \in  \Delta_{k,n} : \bp \in \mathcal{C}_{\delta}( \widehat{\bp} ) \right\}.
\end{eqnarray}
Then 
\begin{eqnarray} \label{eqn:cover_eq}
\bp \in \mathcal{C}_\delta(\widehat{\bp}) \Leftrightarrow  \widehat{\bp}  \in \mathcal{S}(\bp)
\end{eqnarray}
and 
\begin{eqnarray} \label{eqn:anycr}
\mathcal{C}_{\delta} (\widehat{\bp}) = \left\{\bp \in\Delta_{k} :  \widehat{\bp} \in \mathcal{S}({\bp}) \right\}.
\end{eqnarray}
\end{obs}
\noindent We refer to $\mathcal{S}({\bp})$ as a covering collection \cite{chafai2009confidence}, and observe that any confidence region construction can be equivalently expressed in terms of its covering collection according to (\ref{eqn:anycr}).  We also note that for any valid confidence region,  $\mathbb{P}_{\bp} (\mathcal{S}({\bp})) \geq 1-\delta$ holds for all $\bp$,  since $\mathbb{P}_{\bp}  \left( \bp \in \mathcal{C}_{\delta}( \widehat{\bp}) \right)
=    \mathbb{P}_{\bp} \left( \mathcal{S}(\bp)  \right) $ by (\ref{eqn:cover_eq}).

Next we define the \emph{minimal volume confidence region} constructions, which are termed the \emph{level-set} region in \cite{chafai2009confidence}.  The sets are defined in terms of their covering collection.  We note that construction is different from the definition in \cite{chafai2009confidence} to facilitate the main theorem of this paper.  We discuss this difference in Section \ref{sec:dis}.

\begin{defi} \label{def:minvol}
Minimal volume confidence region.  
Let $\mathcal{S^\star({\bp})}: \Delta_{k} \rightarrow \mathcal{P}(\Delta_{k,n})$ be any set valued function that satisfies  
\begin{eqnarray} \label{eqn:s_min}
\mathcal{S}^\star (\bp) = \arg \min_{\left\{\mathcal{S} \in \mathcal{P}(\Delta_{k,n}) : \mathbb{P}_{\bp} ( \mathcal{S}) \geq 1-\delta \right\}} |\mathcal{S}|
\end{eqnarray}
for all $\bp$.  Then the minimal volume confidence region is given as 
\begin{eqnarray} \label{eqn:inv}
\mathcal{C}_{\delta}^\star(\widehat{\bp}) := \left\{\bp \in\Delta_{k} :  \widehat{\bp} \in \mathcal{S}^\star({\bp}) \right\}.
\end{eqnarray}
\end{defi}
\noindent
$\mathcal{S}^\star({\bp})$ is a set valued function, mapping  $\bp$ to a subset of empirical distributions with minimal number of elements among subsets whose probability under $\bp$ equals or exceeds $1-\delta$.   $\mathcal{C}_{\delta}^\star(\widehat{\bp})$ is the subset of the simplex for which the set valued function $\mathcal{S}^\star({\bp})$ includes the observation $\widehat{\bp}$.  

Note that $\mathcal{S}^\star({\bp})$ is in general not unique, and many subsets of $\Delta_{k,n}$ can have minimal cardinality and sufficient probability. As we develop in what follows, any subset of $\Delta_{k,n}$ that satisfies (\ref{eqn:s_min}) must have minimal average volume, and thus, \emph{equal} average volume.   We discuss this in section \ref{sec:dis}. Before proceeding, we note that the construction creates confidence regions with sufficient coverage, by definition.
\begin{obs}
  $\mathcal{C}_{\delta}^\star(\widehat{\bp})$ is a confidence region at level $1-\delta$ since 
$\mathbb{P}_{\bp}  \left( \bp \in \mathcal{C}^\star_{\delta}( \widehat{\bp}) \right)
=    \mathbb{P}_{\bp} \left( \mathcal{S}^\star (\bp)  \right) \geq 1-  \delta$.
\end{obs}

\section{Results}

We next proceed to the main result of the paper, which shows that the confidence set of Definition~\ref{def:minvol}, $C_\delta^\star (\widehat{\bp})$, are on average minimal volume among confidence regions at level $1-\delta$.  
\begin{thm} \label{thm:main}
Let $\mathcal{C}_{\delta}^\star(\widehat{\bp})$ be a confidence region given by Definition~\ref{def:minvol} and define $\vol(\cdot)$ as the Lebesgue measure on the simplex $\Delta_{k,n}$.  Then 
\begin{eqnarray} \nonumber
\sum_{\widehat{\bp} \in \Delta_{k,n}} \vol(\mathcal{C}_{\delta}^\star(\widehat{\bp}))   \leq  \sum_{\widehat{\bp} \in \Delta_{k,n}} \vol(\mathcal{C}_{\delta}(\widehat{\bp} )  )
\end{eqnarray}
for any confidence region $\mathcal{C}_{\delta}(\widehat{\bp})$.
\end{thm}
\begin{proof}  
Note that for any confidence region 
\begin{eqnarray} \label{eqn:crux}
\sum_{\widehat{\bp} \in \Delta_{k,n}} \vol( \mathcal{C}_{\delta} (\widehat{\bp}) )= {\int_{\Delta_k} |\mathcal{S}({\bp})| d\bp} 
\end{eqnarray}
since
\begin{eqnarray*}
\sum_{\widehat{\bp} \in \Delta_{k,n}} \vol( \mathcal{C}_{\delta} (\widehat{\bp}) ) &=& \sum_{\widehat{\bp} \in \Delta_{k,n}} \int_{\mathcal{C}_{\delta} (\widehat{\bp}) } d\bp \\
&=& \sum_{\widehat{\bp} \in \Delta_{k,n}} \int_{ \S_k } \mathbbm{1}_{ \{\bp \in \mathcal{C}_{\delta} (\widehat{\bp}) \} } d\bp \\
&=&  \int_{ \S_k } \sum_{\widehat{\bp} \in \Delta_{k,n}}  \mathbbm{1}_{ \{\bp \in \mathcal{C}_{\delta} (\widehat{\bp}) \} } d\bp \\
&=&  \int_{ \S_k } \left \vert \{ \widehat{\bp} :  \bp \in \mathcal{C}_{\delta} (\widehat{\bp}) \} \right \vert d\bp \\
&=&  {\int_{\S_k} |\mathcal{S}({\bp})| d\bp}
\end{eqnarray*}
where last equality follows from (\ref{eqn:anycr}).
By definition, 
$|\mathcal{S}({\bp})| \geq |\mathcal{S}^\star({\bp})|$ for all $\bp$.  This implies
\begin{eqnarray}
{\int_{\S_{k} } |\mathcal{S}({\bp})| d\bp} \;  \geq { \int_{\S_{k}  } |\mathcal{S}^{\star}({\bp})| d\bp }.
\end{eqnarray}
Given that any confidence region construction can be defined in terms of its covering collection according to Observation \ref{obs:cc}, together with (\ref{eqn:crux}), this implies the result. 
\end{proof}

Theorem \ref{thm:main} shows that, averaged over empirical distributions, the confidence regions defined in (\ref{def:minvol}) have minimal volume. 
The main idea of the proof is to count the sum of the Lebesgue measure of the confidence sets in two ways. The LHS in \eqref{eqn:crux} obtains the sum by adding up the shaded areas corresponding to each point in $\S_{k,n}$. The RHS in \eqref{eqn:crux} obtains the same sum by integrating, over all $\bp \in \Delta_k$, the count of elements in $\Delta_{k,n}$ that include $\bp$ in their confidence region (i.e, integrating the size of the covering collection over $\bp$). 
Fig. \ref{fig:exp} can be used to visualize the steps of the proof. 
We next show that if the multinomial parameter is chosen with uniform probability over the simplex, then the optimal properties of the region still apply.

\begin{prop} 
Let $\bp$ be drawn uniformly at random from $\S_k$ and denote $\mathbb{E}_{\bp} $ expectation with respect to the multinomial parameter $\bp$.  Let $\mathcal{C}_{\delta}^\star(\widehat{\bp})$ be a confidence region construction given by Def. (\ref{def:minvol}), then
\begin{eqnarray} \nonumber
\mathbb{E}_{\bp} \left[ \vol(\mathcal{C}_{\delta}^\star(\widehat{\bp})) \right]  \leq  \mathbb{E}_{\bp} \left[  \vol(\mathcal{C}_{\delta}(\widehat{\bp} )  ) \right] 
\end{eqnarray}
for any confidence region $\mathcal{C}_{\delta}(\widehat{\bp})$.
\end{prop}
\begin{proof}
Suppose $\mathbb{P}_{\bp}(\widehat{\bp}) = \nicefrac{1}{|\S_{k,n}|}$. Then
\begin{align*}
\mathbb{E}_{\bp}[\mu(\mathcal{C}^\star_\delta(\widehat{\bp}))] = \frac{1}{|\S_{k,n}|} \sum_{\widehat{\bp} \in \Delta_{k,n}} \vol(\mathcal{C}_{\delta}^\star(\widehat{\bp}))   \leq \frac{1}{|\S_{k,n}|} \sum_{\widehat{\bp} \in \Delta_{k,n}} \vol(\mathcal{C}_{\delta}(\widehat{\bp} )  ) 
= \mathbb{E}_{\bp} \left[  \vol(\mathcal{C}_{\delta}(\widehat{\bp} )  ) \right], 
\end{align*}
where the inequality is due to Theorem~\ref{thm:main}. Now we show why $\mathbb{P}_{\bp}(\widehat{\bp}) = \nicefrac{1}{|\S_{k,n}|}$. 
A multinomial parameter drawn uniformly at random in $\S_k$ induces a uniform distribution over the set of empirical distributions. This is because the resulting distribution on $\widehat \bp$ is the Dirichlet-Multinomial distribution, or a compound Dirichlet distribution \cite{frigyik2010introduction} with a uniform Dirichlet. 
\end{proof}


As noted in Sec. \ref{sec:pre}, the minimal volume confidence construction is under-specified.  In general there are many covering collections $\mathcal{S}^\star( \bp)$, each of which results in equal and minimal volume confidence regions.  

A simple way to fully specify the confidence regions is to order the empirical distributions based on their probability under $\bp$ (with ties broken randomly), and construct $\mathcal{S}^{\star}(\bp)$ by including the most probable empirical distributions until a mass of $1-\delta$ is obtained.  This results in covering collections that satisfy (\ref{def:minvol}) and also have an additional guarantee on their coverage probability.  We capture this in the following corollary. 
\begin{prop} \label{cor:doublestar}
For any $\bp$, let $\widehat{\bp}_1, \widehat{\bp}_2 \dots$ be an ordering of the elements of $\Delta_{k,n}$ such that $\mathbb{P}_{\bp}(\widehat{\bp}_1 )\geq \mathbb{P}_{\bp}(\widehat{\bp}_2) \geq \dots$, and let $\ell$ be the smallest integer that satisfies 
\begin{eqnarray}
\sum_{i=1}^\ell  \mathbb{P}_{\bp}(\widehat{\bp }_{i}) \geq 1-\delta.
 \end{eqnarray}
Define
$
\mathcal{S}^{\star \star} (\bp) = \{ \widehat{\bp }_{1}, \dots, \widehat{\bp }_{\ell} \}  $
and
$\mathcal{C}_{\delta}^{\star \star} (\widehat{\bp}) := \left\{\bp \in\Delta_{k} :  \widehat{\bp} \in \mathcal{S}^{\star \star} ({\bp}) \right\}$.
Then 
 \begin{eqnarray} \nonumber
 \mathbb{P}_{\bp} ( \bp \in \mathcal{C}_{\delta}^{\star \star}(\widehat{\bp})  )  \geq  \mathbb{P}_{\bp} ( \bp \in \mathcal{C}_{\delta}^\star(\widehat{\bp})  ) \geq 1-\delta 
\end{eqnarray}
holds for all $\bp$.
\end{prop}
\begin{proof}
Since $\mathbb{P}_{\bp} ( \bp \in \mathcal{C}_{\delta}^{\star \star}(\widehat{\bp})  )  = 
\mathbb{P}_{\bp} ( \mathcal{S}^{\star \star}(\bp) )$ by the relationship in (\ref{eqn:cover_eq}), and since $\mathbb{P}_{\bp} (\mathcal{S}^{\star \star}(\bp) ) \geq \mathbb{P}_{\bp} (\mathcal{S}^{\star}(\bp)) $ by the ordering above, the proof follows immediately.
\end{proof}
\noindent
Proposition \ref{cor:doublestar} shows that a particular choice for construction of the covering collection  $\mathcal{S}^{\star \star} (\bp)$ also satisfies a secondary optimality property -- among all confidence regions that have minimal (and equal) average volume, $\mathcal{C}_{\delta}^{\star \star}(\widehat{\bp})$ has maximal coverage probability for all $\bp$.

Proposition \ref{cor:doublestar} highlights the observation that several confidence region constructions have equal average minimal volume.  This occurs because the average is taken over the set of possible empirical distributions.  Provided the minimal cardinality requirement is employed in the construction, the average volume is constant, but the coverage probability may vary. 

Proposition \ref{cor:doublestar} also highlights the difference between the definition of the minimal volume confidence regions defined here, and the level-set construction in \cite{chafai2009confidence}.  In the level-set construction, equiprobable outcomes are either all included or excluded in the covering collections, which precludes the construction from having minimal average volume in this corner case.

\subsection{Confidence Sets for Linear Functionals}

The simplex confidence regions developed above induce optimal confidence sets for linear functionals of the multinomial parameter, such as the mean.  To consider linear functionals we assign numerical values $\{0,1,\dots,k-1\}$ to the vertices of the simplex $\S_k$.   For any 
$\bp\in \S_k$ the mean functional is $m(\bp):=\sum_{i=0}^{k-1} i\, p_i$.
In particular, $\widehat m = m(\widehat p)$ is the empirical mean of $\widehat p$.  Define the confidence set for $\widehat m$ as 
$$\mathcal{C}_{\delta}^\star\big(m(\widehat \bp)\big) = \Big\{m \, : \, \exists \,  \bp\in \mathcal{C}_{\delta}^\star(\widehat{\bp}) \mbox{ such that } m = m(\bp) \Big\} \ . $$
Note that the confidence set $\mathcal{C}_{\delta}^\star\big(m(\widehat \bp)\big)$ depends on $\widehat \bp$, not just the value of the empirical mean $m(\widehat \bp)$. This is crucial since it allows for confidence sets that automatically adapt to distributional characteristics like variance.
Now consider a measure $\mu$ on $\S_k$.  This induces a measure $\mu_m$ on $[0,k-1]$, the range of the mean. Specifically, simply define $\mu_m\big(\mathcal{C}_{\delta}^\star(m(\widehat \bp)) \big):= \mu\big(\mathcal{C}_{\delta}^\star(\widehat{\bp}) \big)$.  Recall, the confidence sets $\mathcal{C}_{\delta}^\star$ have minimum average volume with respect to \emph{any} measure $\mu$.  Therefore, we may choose $\mu$ such that the induced measure on $[0,k-1]$ is uniform. The conclusion is summarized in the following corollary.

\begin{cor}
The confidence sets $\{\mathcal{C}_{\delta}^\star(m(\widehat \bp))\}$ have minimum average Lesbegue measure, where the average is with respect to all possible empirical distributions (depending on $k$ and the sample size $n$).
\end{cor}

\begin{proof}
The result follows since
$$\sum_{\widehat \bp} \mu_m\big(\mathcal{C}_{\delta}^\star\big(m(\widehat \bp)\big)\big) \ = \ \sum_{\widehat \bp} \mu\big(\mathcal{C}_{\delta}^\star\big(\widehat \bp\big)\big) \ . $$
\end{proof}

\begin{ex}
The following measure on $\S_3$ induces uniform (Lesbegue) measure on $[0,2]$. The measure is a mixture distribution defined as follows. Let $u \sim \mbox{uniform}[-1,1]$.  If $u\geq 0$, then set $p_0=u$ and $p_2=0$, otherwise set $p_2=-u$ and $p_0=0$. Finally, set $p_1=1-p_0-p_2$. This defines a measure $\mu$ on $\S_3$ such that $m=p_1+2p_2 \sim \mbox{uniform}[0,2]$.
\end{ex}

The results above can be generalized to arbitrary linear functionals and non-uniform induced measures (if desirable).  It may be possible to use the same approach to construct confidence sets for nonlinear functions, but this is left to future work.

\section{Discussion and Extensions} \label{sec:dis}

\subsection{Relationship to Significance Testing}
The confidence regions presented in this paper and in \cite{chafai2009confidence} are closely related to $p$-values in statistical significance testing.  Often, the phrase $p$-value is used to describe an approximate $p$-value based on a normal approximation.  A more precise interpretation of a $p$-value can be related to the construction of $\mathcal{C}_\delta(\widehat{\bp})$.
\begin{defi} $p$-value. The $p$-value of an outcome $\widehat{\bp}$ (under the hypothesis $\bp$) is:
\begin{eqnarray} \nonumber
p(\widehat{\bp}; \bp) = \sum_{\widehat{\bq} \in \Delta_{k,n} : \mathbb{P}_{\bp}(\widehat{\bq}) \leq \mathbb{P}_{\bp}(\widehat{\bp})  } \mathbb{P}_{\bp}  \left( \widehat{\bq}  \right).
\end{eqnarray}
\end{defi}
\noindent 
A $p$-value has the following interpretation in statistical significance testing: $p$ is the \emph{probability that the observed outcome or something less probable} occurred under the hypothesis $\bp$.   A small $p$-value corresponds to a strange outcome under the null, and thus corresponds to rejection of the null hypothesis.  The level-set confidence regions described in this paper and in \cite{chafai2009confidence} can be stated in terms of covering collection based on $p$-values: $\mathcal{C}_\delta(\bp) = \{ \bp : p(\widehat{\bp}; \bp ) > \delta \}$.

We note that the level-set confidence regions and their expressions herein are closely related to `exact' confidence regions defined in \cite{blyth1983binomial} for the specific case when $k=2$.  The confidence region defined by an exact test is the \emph{range of parameters over which the outcome is not statistically significant at a p-value of $1-\delta$}.  Extending this to the multinomial setting is the essence of the level-set confidence regions.

\subsection{Relationship to Sanov Confidence Regions}
Sanov's theorem (Theorem 11.4.1 in \cite{cover2012elements}) allows us to bound the probability of observing a set of empirical distributions using its Kullback Leibler distance to the data-generating distribution. Since the statement of the theorem involves an infimum over Kullback Leibler distances, we can use it to obtain the following inequality:
\begin{eqnarray*} \nonumber
\mathbb{P}_{\bp}(\mathrm{KL}(\widehat{\bp}, \bp) > z) \leq (n+1)^ke^{-nz}  \\
\end{eqnarray*}
which implies
\begin{eqnarray*}
\mathbb{P}_{\bp}\left(\mathrm{KL}(\widehat{\bp}, \bp) \leq \frac{\log((n+1)^k / \delta)}{n}\right) \geq 1-\delta
\end{eqnarray*}
where 
\begin{equation} \nonumber
\mathrm{KL}(\bp, \bp') := \sum_{i=1}^{k} p_i \log\left(\frac{p_i}{p_i'} \right)
\end{equation}
is the Kullback Leibler divergence. One can view the previous inequality as a concentration result for the Kullback Leibler divergence between the observed empirical distribution and the true distribution. The work done in \cite{mardia2018concentration}  has sharpened these types of results in several parameter ranges. For example, when $k \leq e\sqrt[3]{n/8\pi}$, \cite{mardia2018concentration} shows that
\begin{align*} \nonumber
&\mathbb{P}_{\bp}(\mathrm{KL}(\widehat{\bp}, \bp) > z) \leq 2(k-1)e^{-nz/(k-1)}\\  
\end{align*}
implies
\begin{align}
\label{eqn:mer_kl}
\mathbb{P}_{\bp}\left(\mathrm{KL}(\widehat{\bp}, \bp) \leq (k-1)\frac{\log(2(k-1) / \delta)}{n}\right) \geq 1-\delta.
\end{align}
Thus using Sanov's theorem gives us a choice for a confidence region of level $1-\delta$. Another approach used by \cite{nowak2019tighter} to obtain a confidence region is to obtain bounds on the marginal probabilities $\{p_i : i \in \{1, 2, \ldots, k\}\}$. This can be done as $n\hat{p}_i$ corresponds to $n$ i.i.d.\ realizations of a Bernoulli random variable having mean as $p_i$. By allocating $\delta/k$ error probability in bounding each of the marginal parameters, we get using the Bernoulli-KL inequality \cite{garivier2011kl} that for each $i \in \{1, 2, \ldots, k\}$
\begin{align} \label{eqn:bern_kl}
\mathbb{P}_{p_i}(\mathrm{KL}([\hat{p}_i, 1 - \hat{p}_i], [p_i, 1-p_i]) > z)
\leq 2e^{-nz} 
\end{align}
which implies 
\begin{align} \nonumber
\mathbb{P}_{\bp}\left( \bigcap_i \mathrm{KL}([\hat{p}_i, 1 - \hat{p}_i], [p_i, 1-p_i]) \leq \frac{\log(2k/\delta)}{n} \right)
\geq 1-\delta.
\end{align}

Both (\ref{eqn:mer_kl}) and (\ref{eqn:bern_kl}) give us valid confidence regions for the multinomial parameter. We plot these regions along with the proposed region in Figure~\ref{fig:singleCB} in Sec. \ref{sec:numerical}.

\subsection{Computation}\label{sec:exp}



Computation of $\mathcal{C}_\delta^{\star} (\widehat{\bp})$ requires enumerating all empirical outcomes and computing partial sums.  
In our experiments, enumerating and ordering the empirical distributions for $k=5$ and $n=50$ and checking membership in $\mathcal{C}_\delta^{\star} (\widehat{\bp})$ completes in around two seconds on a modern laptop.  Regardless, as computation scales as $n^{k}$, computation of membership in $\mathcal{C}_\delta^{\star} (\widehat{\bp})$ becomes  prohibitive for a modest number of categories.
We note that the large sample regime, which is not the focus of the work here, is served well by traditional confidence regions based on asymptotic statistics.

There are a number of ways in which computation of the proposed confidence regions can be accelerated.  First, in the numerical experiments, we use the approximate $p$-values returned by Pearson's $\chi^2$ test to obtain a course estimate of the confidence regions, and refine it using exhaustive computation only when needed. 

Next, to further aid in computation, we show an outer bound based on the Kullback Leibler divergence that can be used to accelerate computation of the regions.  The bound provides a way to confirm if a particular $\bp$ is outside $\mathcal{C}_\delta^{\star} (\widehat{\bp})$.  
\begin{thm}
\emph{Outer bound.} The following inequality holds:
\begin{eqnarray} \nonumber
p(\widehat{\bp};\bp) \leq  (n+1)^{2k} \exp\left(-n \ \mathrm{KL}(\widehat{\bp}, \bp)   \right)
\end{eqnarray}
\end{thm}


\begin{proof}
From \cite{cover2012elements} (Theorem 11.1.4),  we can bound the probability of any empirical distribution $\widehat{\bq}$ under $\bp$:
\begin{align} \label{eqn:typebnd} 
\frac{1}{(n+1)^k} \exp\left(-n \mathrm{KL}(\widehat{\bq},\bp) \right) 
\leq \mathbb{P}_{\bp}(\widehat{\bq})
\leq \exp \left(-n \mathrm{KL}(\widehat{\bq},\bp) \right).
\end{align}
Thus, for any $\mathbb{P}_{\bp}(\widehat{\bq}) \leq \mathbb{P}_{\bp} (\widehat{\bp})$, 
\begin{align*}
\frac{1}{(n+1)^k} \exp\left(-n \mathrm{KL}(\widehat{\bq},\bp)\right) \leq \exp\left(-n \mathrm{KL}(\widehat{\bp},\bp)\right)
\end{align*}
which implies the following. 
Let $\mathcal{S} \subset \Delta_{k,n}$ be a set of empirical distributions that satisfies $ \mathbb{P}_{{\bp}}(\widehat{\bq}) \leq \mathbb{P}_{{\bp}}(\widehat{\bp})$ for all $\widehat{\bq} \in \mathcal{S}$.
Then, 
\begin{align}\label{thm:bnd1}
\min_{ \widehat{\bq} \in \mathcal{S} }  \mathrm{KL}(\widehat{\bq},\bp)
 \geq \mathrm{KL}(\widehat{\bp},\bp) - \frac{k}{n} \log (n+1).
\end{align}
Next, we require Sanov's Theorem, \cite{cover2012elements} (Theorem 11.4.1), which states the following.  Let $\mathcal{S} \subset \Delta_{k,n}$ be a set of empirical distributions. 
Then
\begin{align} \label{eqn:sanov}
\mathbb{P}_{\bp}(\mathcal{S}) \leq  (n+1)^k \exp\left(-n  \min_{\widehat{\bq} \in \mathcal{S}}  \mathrm{KL}(\widehat{\bq},\bp)\right).
\end{align}
Choosing $\mathcal{S} =\{\widehat{\bq} \in \Delta_{k,n} : \mathbb{P}_{\bp}(\widehat{\bq}) \leq \mathbb{P}_{\bp}(\widehat{\bp}) \}$ and combining \eqref{thm:bnd1} and (\ref{eqn:sanov}), we conclude
\begin{align*}
p(\widehat{\bp};\bp) = 
\sum_{\widehat{\bq} \in \Delta_{k,n} : \mathbb{P}_{\bp}(\widehat{\bq}) \leq \mathbb{P}_{\bp}(\widehat{\bp})     } \mathbb{P}_{\bq } \left( \widehat{\bq}  \right)
 \leq  (n+1)^{2k} e^{\left(-n \mathrm{KL}(\widehat{\bp},\bp)   \right)} .
\end{align*}
\end{proof}

Note that the above bound has and additional factor of two in the second term, beyond what arises from directly inverting Sanov's Theorem \cite{cover2012elements}.  This arises from the fact that $ \widehat{\bp}$ is not necessarily the minimal empirical distribution in KL divergence, i.e, it is not necessary true that $\widehat{\bp}$ equals  
\begin{eqnarray}
\arg \min_{\{\widehat{\bq} \in \Delta_{k,n} : \mathbb{P}_{\bp}(\widehat{\bq}) \leq \mathbb{P}_{\bp} (\widehat{\bp}) \}} \mathrm{KL}(\widehat{\bq},\bp).
\end{eqnarray}

\subsection{Numerical Experiments} \label{sec:numerical}

\begin{figure}
\centering
\includegraphics[width=0.60\textwidth]{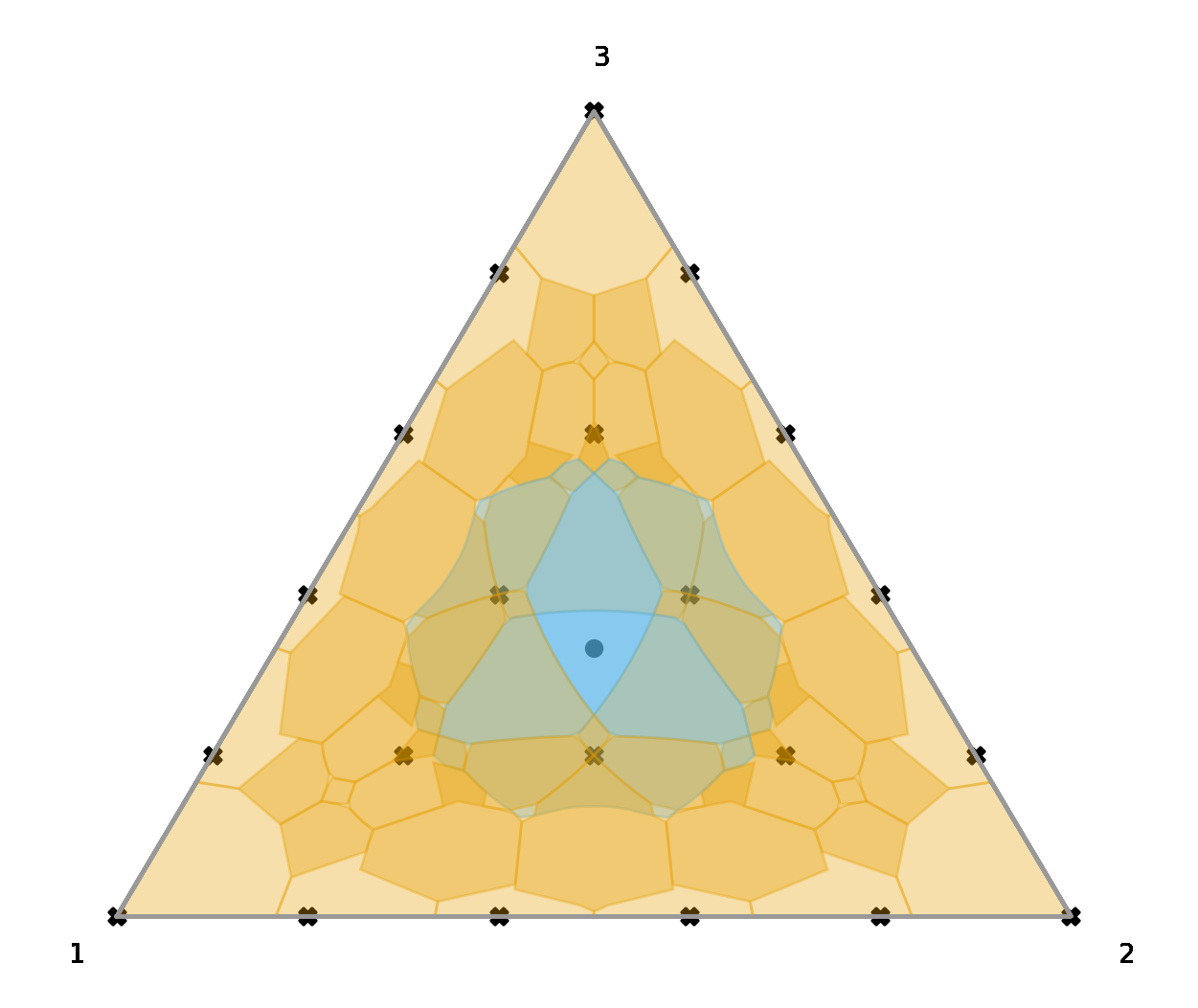}
\caption{All confidence regions $\{\mathcal{C}^{\star \star}_{0.7}(\widehat{\bp}) : \widehat{\bp} \in \Delta_{3,5}\}$ are shaded over a picture of the three dimensional simplex. 
The figure depicts the 3-simplex with black crosses indicating the empirical proportions that could be observed in 5 trials (their total number is $\binom{5+3-1}{3-1}=21$). The number of regions that cover a parameter in $\Delta_3$ vary based on where the parameter lies within the simplex. 
As an example, the uniform parameter $\bp_{u} = [\nicefrac{1}{3}, \nicefrac{1}{3}, \nicefrac{1}{3}]$ is shown by a blue dot in the center of the simplex.  $\bp_u$ is covered by the confidence regions of three empirical distributions: $\mathcal{S}^{\star \star}(\bp_u) = \{ [\nicefrac{1}{5}, \nicefrac{2}{5}, \nicefrac{2}{5}], [\nicefrac{2}{5}, \nicefrac{1}{5}, \nicefrac{2}{5}], [\nicefrac{2}{5}, \nicefrac{2}{5}, \nicefrac{1}{5}] \}$.  The confidence regions associated with these three empirical distributions are indicated in blue. 
The main idea in the proof of Thm. \ref{thm:main} is to count the sum of Lebesgue measure of the confidence sets in two ways. The LHS in \eqref{eqn:crux} obtains the sum by adding up the shaded areas corresponding to each point in $\S_{3,5}$. The RHS in \eqref{eqn:crux} obtains the same sum by integrating, over all $\bp \in \S_3$, the count of elements in $\S_{3,5}$ that include $\bp$ in their confidence region (i.e, integrating the size of the covering collection over $\bp$). 
\label{fig:exp}}
\end{figure}

We begin with a visualization of the proposed confidence regions $\mathcal{C}^{\star \star}_\delta(\widehat{\bp})$ for a small scale experiment with $n = 5$ samples of a $k = 3$ categorical random variable. 
Figure~\ref{fig:exp} shows the confidence regions at level $1 - \delta = 0.3$ for all possible empirical distributions in the discrete simplex $\Delta_{3,5}$ overlaid on top of each other. We also show the uniform parameter $[\nicefrac{1}{3}, \nicefrac{1}{3}, \nicefrac{1}{3}] \in \Delta_3$ and indicate the regions that include it at the chosen confidence level, i.e., its covering collection. In this example, from the figure, we can see that $|\mathcal{S}^{\star \star}([\nicefrac{1}{3}, \nicefrac{1}{3}, \nicefrac{1}{3}])| = 3$.

Next, in Fig. \ref{fig:singleCB}, we show an illustration of the proposed region contrasted with the Sanov and polytope confidence regions of \eqref{eqn:mer_kl} and \eqref{eqn:bern_kl} for a different set of problem parameters.  The illustration highlights the significant difference in volume of the proposed region when compared against the Sanov and polytope regions. 

\begin{figure}
\centering
\includegraphics[width=0.50\textwidth]{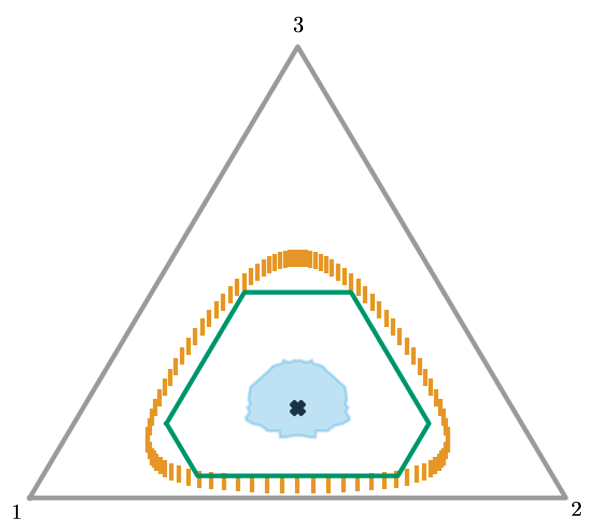}
\caption{Proposed confidence region (Proposition~\ref{cor:doublestar}) shown in blue with the Sanov confidence region \eqref{eqn:mer_kl} in orange and the polytope confidence region \eqref{eqn:bern_kl} in green. The black cross is the observed empirical distribution $\widehat{\bp} = [\nicefrac{6}{15}, \nicefrac{6}{15}, \nicefrac{3}{15}]$ of $15$ realizations of a categorical random variable. All confidence regions are shown at $30\%$ confidence level.
\label{fig:singleCB}}
\end{figure}

In Fig. \ref{fig:widths}, we illustrate the power of the level-set construction for linear functionals by compare confidence intervals for the mean. The classical Chernoff bound and Hoeffding's inequality are standard textbook examples that bound deviations of the empirical mean from the true mean.  These are sometimes useful in algorithm analysis, but often too loose in practice \cite{langford2005tutorial}, since they essentially assume the worst case variance.  Refinements such as the KL-Bernoulli  bound \cite{langford2005tutorial,garivier2011kl} can be significantly better, especially in cases where the true mean is close to the extremes, e.g.,  $0$ or $1$ in the case of random variable in $[0,1]$. These bounds have shown theoretical and empirical improvement in multi-armed bandit algorithms \cite{garivier2011kl,tanczos2017kl}. Bernstein's inequality offers  potential for improvement, by taking the underlying  scale/variance into account.  The empirical Bernstein bound \cite{mnih2008empirical,maurer2009empirical,peel2010empirical,audibert2009exploration,balsubramani2016sequential}
uses an estimate of the variance to tighten confidence intervals on the mean.  For sufficiently large sample sizes, this bound can be significantly better than those mentioned above, showing that additional information about the shape of the distribution can be helpful in improving bounds. The empirical Bernstein bound is quite loose in small sample regimes, which significantly reduces its practicality.  

The level-set construction proposed in this paper can require several times fewer samples to achieve a specific confidence interval width when compared with the approaches described above. This implies that the sample complexity or regret of bandit and reinforcement learning algorithms can be reduced by a similar factor \cite{tanczos2017kl}.
We demonstrate this by plotting the widths of these methods with increasing sample size in Figure \ref{fig:widths} .

\begin{figure}
\centering
\includegraphics[width=0.54 \textwidth]{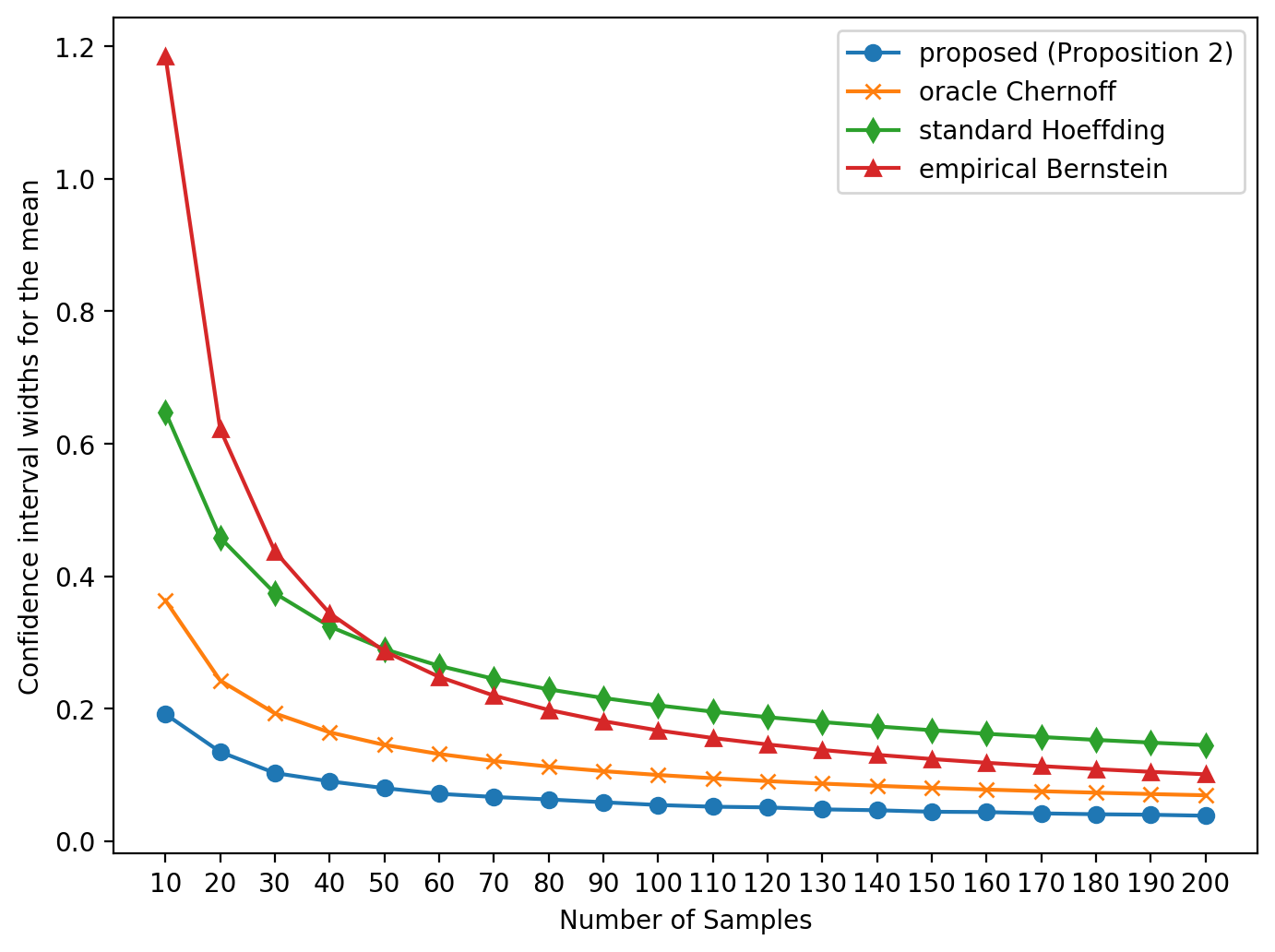}
\caption{Confidence interval widths vs.\ sample size $n$ for $\widehat \bp = (n/10,n/10,8n/10)$ for the mean functional that maps $p_1$ to $0$, $p_2$ to $1/2$, and $p_3$ to $1$.   The level-set construction of confidence intervals is compared to intervals based on subGaussian/Chernoff and empirical Bernstein bounds.  The ``oracle'' Chernoff interval is based on the true variance, rather than the worst-case upper bound. This essentially illustrates the best possible interval one could obtain using sub-Gaussian tail bounds (given perfect knowledge of the scale). The gap between its width and that of the level-set method is due to Markov's inequality in the Chernoff bound. All confidence intervals are at $30\%$ confidence level. \label{fig:widths}}
\end{figure}

We demonstrate an experiment where using the proposed confidence regions allows us to identify the best among five arms using fewer samples than other baseline methods. The arms are $3$-category pmfs, shown in the table below.
\begin{align*}
\begin{matrix}
\text{Arms} & 1\text{-star} & 2\text{-star} & 3\text{-star} \\
1 & 0.1 & 0.6 & 0.3 \\
2 & 0.3 & 0.6 & 0.1\\
3 & 0.4 & 0.5 & 0.1\\
4 & 0.6 & 0.3 & 0.1\\
5 & 0.7 & 0.2 & 0.1
\end{matrix}
\end{align*}
We run the LUCB algorithm \cite{Kaufmann13information} with tolerance level $0$ and confidence level $0.95$. We used three different methods of constructing the upper and lower confidence bounds. Figure~\ref{fig:boxplot} summarizes the stopping times for each of those confidence bounds. Note that the sampling and stopping strategies are the same for each of the three cases, and the improvement in number of samples required is solely due to the tighter confidence regions constructed using our proposed method.

To aid the computation of our confidence regions, we first computed an approximation to the optimal confidence regions using the p-values returned by a $\chi^2$ test. We only computed the optimal confidence region if the $\chi^2$ test indicated that one arm had a higher mean than the other at the desired confidence level. This allows us to speed up the computation while continuing to have the theoretical guarantees of our proposed confidence regions.
\begin{figure}
    \centering
    \includegraphics[width=0.66\textwidth]{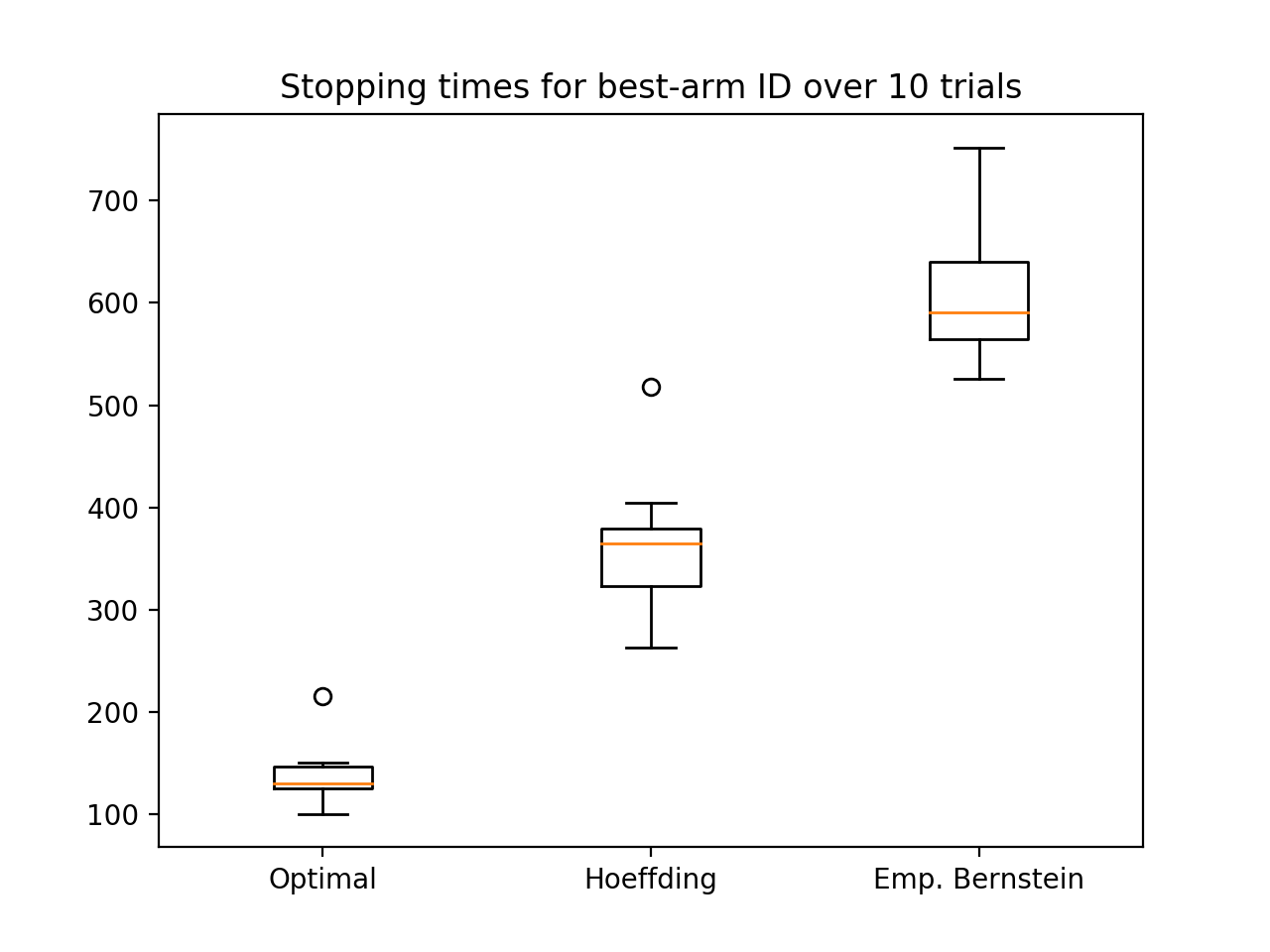}
    \caption{Box plots summarizing the stopping times over ten trials of the best arm identification experiment. The lower and upper levels of the boxes indicate $25$th and $75$th percentile of the data. The hollow circles indicate outliers.}
    \label{fig:boxplot}
\end{figure}

\section{Summary}

Construction of tight confidence regions is a challenging problem with a long history.  The problem has seen increased interest, as confidence bounds are central to the analysis and operation of many learning algorithms, especially sequential methods such as active learning, bandit problems, and reinforcement learning.  

This paper shows an optimal construction for confidence regions for the parameter of a multinomial distribution. The sets, termed \emph{minimal volume confidence regions} or level-set regions \cite{chafai2009confidence}, are optimal in the sense of having minimal volume in the probability simplex, on average, for a prescribed coverage (i.e., confidence).  More precisely, when averaged across the possible empirical outcomes or a uniform prior on the unknown parameter $\bp$, the regions have minimal volume among any confidence region construction that satisfies the coverage guarantee.  The \emph{minimal volume confidence regions} or level-set regions \cite{chafai2009confidence} are a generalization of the famous Clopper-Pearson confidence interval for the binomial \cite{clopper1934use}.  Clopper-Pearson, exact, and confidence regions are closely related to statistical significance testing. 

The minimal volume confidence regions may have utility in a broad range of applications.  Confidence regions not only play a central role in the analysis and design of modern machine learning algorithms, include sequential and adaptive methods such as multi-armed bandits and reinforcement learning, but traditional testing problems such as A/B testing.  An additional contribution of this paper is to show that the minimal volume confidence regions induce optimal (minimum-length) confidence intervals for linear functionals, such as the mean.  Hence, the induced confidence intervals are tighter, on average, than any known constructions, including Hoeffding bounds, Kullback Leibler divergence-based bounds \cite{garivier2013}, and the empirical Bernstein bound \cite{mnih2008empirical,maurer2009empirical,audibert2009exploration}.  To achieve a desired interval width, the new bounds require several times fewer samples than standard bounds in many cases.  This implies that the sample complexity or regret of bandit and reinforcement learning can be reduced by a corresponding factor.  

While computation of the regions is possible for modest $n$ and $k$, it can become prohibitive for problems with a large number of categories and samples.  To aid in computation, we relate the regions to $p$ values, and derive a bound based on Kullback Leibler divergence that can be used to accelerate computation.  In this paper we focused our attention on the multinomial parameter due to its wide applicability and importance across adaptive machine learning.  We note that the techniques can be extended to more general measure spaces equipped with a conditional probability measure, which we leave for future work. 

\bibliographystyle{IEEEtran}
\bibliography{confidence,multistar}

\begin{thebibliography}{10}
\providecommand{\url}[1]{#1}
\csname url@samestyle\endcsname
\providecommand{\newblock}{\relax}
\providecommand{\bibinfo}[2]{#2}
\providecommand{\BIBentrySTDinterwordspacing}{\spaceskip=0pt\relax}
\providecommand{\BIBentryALTinterwordstretchfactor}{4}
\providecommand{\BIBentryALTinterwordspacing}{\spaceskip=\fontdimen2\font plus
\BIBentryALTinterwordstretchfactor\fontdimen3\font minus
  \fontdimen4\font\relax}
\providecommand{\BIBforeignlanguage}[2]{{%
\expandafter\ifx\csname l@#1\endcsname\relax
\typeout{** WARNING: IEEEtran.bst: No hyphenation pattern has been}%
\typeout{** loaded for the language `#1'. Using the pattern for}%
\typeout{** the default language instead.}%
\else
\language=\csname l@#1\endcsname
\fi
#2}}
\providecommand{\BIBdecl}{\relax}
\BIBdecl

\bibitem{clopper1934use}
C.~J. Clopper and E.~S. Pearson, ``The use of confidence or fiducial limits
  illustrated in the case of the binomial,'' \emph{Biometrika}, vol.~26, no.~4,
  pp. 404--413, 1934.

\bibitem{nowak2019tighter}
R.~Nowak and E.~T\`anczos, ``Tighter confidence intervals for rating systems,''
  2019.

\bibitem{chafai2009confidence}
D.~Chafai and D.~Concordet, ``Confidence regions for the multinomial parameter
  with small sample size,'' \emph{Journal of the American Statistical
  Association}, vol. 104, no. 487, pp. 1071--1079, 2009.

\bibitem{agresti1998approximate}
A.~Agresti and B.~A. Coull, ``Approximate is better than “exact” for
  interval estimation of binomial proportions,'' \emph{The American
  Statistician}, vol.~52, no.~2, pp. 119--126, 1998.

\bibitem{agresti2003dealing}
A.~Agresti, ``Dealing with discreteness: making exact confidence intervals for
  proportions, differences of proportions, and odds ratios more exact,''
  \emph{Statistical Methods in Medical Research}, vol.~12, no.~1, pp. 3--21,
  2003.

\bibitem{garivier2013}
A.~{Garivier}, ``Informational confidence bounds for self-normalized averages
  and applications,'' in \emph{2013 IEEE Information Theory Workshop (ITW)},
  2013, pp. 1--5.

\bibitem{mnih2008empirical}
V.~Mnih, C.~Szepesv{\'a}ri, and J.-Y. Audibert, ``Empirical bernstein
  stopping,'' in \emph{Proceedings of the 25th international conference on
  Machine learning}.\hskip 1em plus 0.5em minus 0.4em\relax ACM, 2008, pp.
  672--679.

\bibitem{maurer2009empirical}
A.~Maurer and M.~Pontil, ``Empirical bernstein bounds and sample variance
  penalization,'' in \emph{COLT 2009-The 22nd Conference on Learning Theory},
  2009.

\bibitem{audibert2009exploration}
J.-Y. Audibert, R.~Munos, and C.~Szepesv{\'a}ri, ``Exploration--exploitation
  tradeoff using variance estimates in multi-armed bandits,'' \emph{Theoretical
  Computer Science}, vol. 410, no.~19, pp. 1876--1902, 2009.

\bibitem{jamieson2013finding}
K.~Jamieson, M.~Malloy, R.~Nowak, and S.~Bubeck, ``On finding the largest mean
  among many,'' \emph{arXiv preprint arXiv:1306.3917}, 2013.

\bibitem{malloy2014sequential}
M.~L. Malloy and R.~D. Nowak, ``Sequential testing for sparse recovery,''
  \emph{IEEE Transactions on Information Theory}, vol.~60, no.~12, pp.
  7862--7873, 2014.

\bibitem{auer2002finite}
P.~Auer, N.~Cesa-Bianchi, and P.~Fischer, ``Finite-time analysis of the
  multiarmed bandit problem,'' \emph{Machine learning}, vol.~47, no. 2-3, pp.
  235--256, 2002.

\bibitem{jamieson2014lil}
K.~Jamieson, M.~Malloy, R.~Nowak, and S.~Bubeck, ``lil’ucb: An optimal
  exploration algorithm for multi-armed bandits,'' in \emph{Conference on
  Learning Theory}, 2014, pp. 423--439.

\bibitem{frigyik2010introduction}
B.~A. Frigyik, A.~Kapila, and M.~R. Gupta, ``Introduction to the dirichlet
  distribution and related processes,'' \emph{Department of Electrical
  Engineering, University of Washington, UWEETR-2010-0006}, no. 0006, pp.
  1--27, 2010.

\bibitem{blyth1983binomial}
C.~R. Blyth and H.~A. Still, ``Binomial confidence intervals,'' \emph{Journal
  of the American Statistical Association}, vol.~78, no. 381, pp. 108--116,
  1983.

\bibitem{cover2012elements}
T.~M. Cover and J.~A. Thomas, \emph{Elements of information theory}.\hskip 1em
  plus 0.5em minus 0.4em\relax John Wiley \& Sons, 2012.

\bibitem{mardia2018concentration}
J.~Mardia, J.~Jiao, E.~T{\'a}nczos, R.~D. Nowak, and T.~Weissman,
  ``Concentration inequalities for the empirical distribution,'' \emph{arXiv
  preprint arXiv:1809.06522}, 2018.

\bibitem{garivier2011kl}
A.~Garivier and O.~Capp{\'e}, ``The kl-ucb algorithm for bounded stochastic
  bandits and beyond,'' in \emph{Proceedings of the 24th annual Conference On
  Learning Theory}, 2011, pp. 359--376.

\bibitem{langford2005tutorial}
J.~Langford, ``Tutorial on practical prediction theory for classification,''
  \emph{Journal of machine learning research}, vol.~6, no. Mar, pp. 273--306,
  2005.

\bibitem{tanczos2017kl}
E.~Tanczos, R.~Nowak, and B.~Mankoff, ``A kl-lucb algorithm for large-scale
  crowdsourcing,'' in \emph{Advances in Neural Information Processing Systems},
  2017, pp. 5894--5903.

\bibitem{peel2010empirical}
T.~Peel, S.~Anthoine, and L.~Ralaivola, ``Empirical bernstein inequalities for
  u-statistics,'' in \emph{Advances in Neural Information Processing Systems},
  2010, pp. 1903--1911.

\bibitem{balsubramani2016sequential}
A.~Balsubramani and A.~Ramdas, ``Sequential nonparametric testing with the law
  of the iterated logarithm,'' in \emph{Proceedings of the Thirty-Second
  Conference on Uncertainty in Artificial Intelligence}.\hskip 1em plus 0.5em
  minus 0.4em\relax AUAI Press, 2016, pp. 42--51.

\bibitem{Kaufmann13information}
\BIBentryALTinterwordspacing
E.~Kaufmann and S.~Kalyanakrishnan, ``Information complexity in bandit subset
  selection,'' ser. Proceedings of Machine Learning Research, S.~Shalev-Shwartz
  and I.~Steinwart, Eds., vol.~30.\hskip 1em plus 0.5em minus 0.4em\relax
  Princeton, NJ, USA: PMLR, 12--14 Jun 2013, pp. 228--251. [Online]. Available:
  \url{http://proceedings.mlr.press/v30/Kaufmann13.html}
\BIBentrySTDinterwordspacing

\end{thebibliography}

\end{document}